%% file: New_IEEEtran_how-to.tex
\documentclass[lettersize,journal]{IEEEtran}

\def\BibTeX{{\rm B\kern-.05em{\sc i\kern-.025em b}\kern-.08em
T\kern-.1667em\lower.7ex\hbox{E}\kern-.125emX}}

\input{packages.tex}

\begin{document}

\title{Exploring Information-Theoretic Metrics Associated with Neural Collapse in Supervised Training}
\author{Kun Song~$^*$, Zhiquan Tan~$^*$, Bochao Zou~$^\dagger$, Jiansheng Chen, \textit{Senior Member, IEEE}, \\ Huimin Ma~$^\dagger$, \textit{Senior Member, IEEE}, and Weiran Huang~$^\dagger$
\thanks{$^\dagger$ Correspondence author. $^*$ Equal Contribution. \par Kun Song, Bochao Zou, Jiansheng Chen and Huimin Ma are with the School of Computer and Communication Engineering, University of Science and Technology Beijing, Beiing 100083, China (e-mail: songkun@xs.ustb.edu.cn; zoubochao@ustb.edu.cn; jschen@ustb.edu.cn; mhmpub@ustb.edu.cn). \par  Zhiquan Tan is with the Department of Mathematical Seiences, Tsinghua University, Beijing 100084, China (e-mail: tanzq21@mails.tsinghua.edu.cn). \par Weiran Huang is with the Qing Yuan Research Institute, SEIEE, Shanghai Jiao Tong University, Shanghai 200240, China and Shanghai AI Laboratory, Shanghai 200232,  China (weirang.huang@outlook.com)}}

\markboth{Journal of \LaTeX\ Class Files,~Vol.~18, No.~9, September~2020}
{How to Use the IEEEtran \LaTeX \ Templates}

\maketitle

\input{content/abstract.tex}

\input{content/intro.tex}

\input{content/related.tex}

\input{content/prelim.tex}

\input{content/method.tex}

\input{content/MESL.tex}

\input{content/IRSL.tex}

\input{content/FSFT}

\input{content/SSL.tex}

\input{content/conclusion.tex}

\bibliographystyle{IEEEtran}
\bibliography{reference}

\end{document}

%% file: packages.tex
\usepackage{microtype}
\usepackage{amsmath}
\usepackage{graphicx}
\usepackage{booktabs} %
\usepackage{xcolor}
\usepackage{subcaption}

\usepackage{array}
\usepackage{textcomp}
\usepackage{stfloats}
\usepackage{url}
\usepackage{verbatim}
\usepackage{balance}

\newcommand{\tableref}[1]{TABLE~\ref{#1}}
\newcommand{\figureref}[1]{Fig.~\ref{#1}}
\newcommand{\defineref}[1]{Definition~\ref{#1}}
\newcommand{\theoremref}[1]{Theorem~\ref{#1}}
\newcommand{\lemaref}[1]{Lemma~\ref{#1}}
\newcommand{\secref}[1]{Section~\ref{#1}}
\newcommand{\corollaryref}[1]{Corollary~\ref{#1}}
\newcommand{\algref}[1]{Algorithm~\ref{#1}}

\usepackage{pgfplots}
\pgfplotsset{compat = newest}
\usepackage{multirow}

\usepackage{framed}
\usepackage{tcolorbox}

\usepackage{amssymb}
\usepackage{mathrsfs}
\usepackage{mathtools}
\usepackage{amsthm}

\DeclareMathOperator*{\argmin}{arg\,min}

\usepackage[linesnumbered,ruled,vlined]{algorithm2e}
\usepackage{algorithmic}
\usepackage{listings}
\usepackage{makecell}
\usepackage{colortbl}
\usepackage{color}
\usepackage{wrapfig}
\usepackage{cancel}
\usepackage{soul,xcolor}
\usepackage{pifont}
\usepackage{tikz}
\usetikzlibrary{shapes, positioning, arrows.meta, calc, decorations.pathmorphing, quotes}

\SetKwInput{KwInit}{Initialize}

\theoremstyle{plain}
\newtheorem{theorem}{Theorem}[section]

\newtheorem{lemma}[theorem]{Lemma}
\newtheorem{corollary}[theorem]{Corollary}
\theoremstyle{definition}
\newtheorem{definition}[theorem]{Definition}

\theoremstyle{remark}

\newcommand{\alink}[1]{\href{#1}{paper-link}}

\usepackage{enumitem}

\definecolor{citecolor}{HTML}{0071BC}
\definecolor{linkcolor}{HTML}{ED1C24}

\definecolor{commentcolor}{RGB}{110,154,155}   

\usepackage{cite}

\newcommand{\citep}[1]{\cite{#1}}
\usepackage{xcolor}         %
\definecolor{citecolor}{HTML}{0071BC}
\definecolor{linkcolor}{HTML}{D32F2F}

\usepackage[pagebackref=true,breaklinks=true,colorlinks,bookmarks=false]{hyperref}
\hypersetup{colorlinks=true, linkcolor=linkcolor, citecolor=citecolor,urlcolor=black}

%% file: content/abstract.tex
\begin{abstract}

In this paper, we introduce matrix entropy as an analytical tool for studying supervised learning, investigating the information content of data representations and classification head vectors, as well as the dynamic interactions between them during the supervised learning process. Our experimental results reveal that matrix entropy effectively captures the variations in information content of data representations and classification head vectors as neural networks approach Neural Collapse during supervised training, while also serving as a robust metric for measuring similarity among data samples. Leveraging this property, we propose Cross-Model Alignment (CMA) loss to optimize the fine-tuning of pretrained models. To characterize the dynamics of neural networks nearing the Neural Collapse state, we introduce two novel metrics: the Matrix Mutual Information Ratio (MIR) and the Matrix Entropy Difference Ratio (HDR), which quantitatively assess the interactions between data representations and classification heads in supervised learning, with theoretical optimal values derived under the Neural Collapse state. Our experiments demonstrate that MIR and HDR effectively explain various phenomena in neural networks, including the dynamics of standard supervised training, linear mode connectivity. Moreover, we use MIR and HDR to analyze the dynamics of grokking, which is a fascinating phenomenon in supervised learning where a model unexpectedly exhibits generalization long after achieving training data fit. Additionally, we employ mutual information and entropy difference as loss terms in supervised and semi-supervised learning to optimize the information interactions between samples and classification heads. Empirical results validate the efficacy of these methods, showcasing that MIR and HDR not only provide deeper insights into the training process but also enhance the overall training performance.

\end{abstract}


%% file: content/intro.tex
\section{Introduction}

\IEEEPARstart{S}{upervised} learning is a cornerstone of machine learning, with its roots tracing back to the early days of artificial intelligence. By leveraging large-scale annotated datasets such as ImageNet~\cite{krizhevsky2012imagenet} and COCO~\cite{lin2014microsoft}, supervised learning has achieved remarkable success in tasks like image recognition~\cite{he2016deep, girshick2015fast, ronneberger2015u}, natural language processing~\cite{vaswani2017attention}, and speech recognition~\cite{hinton2012deep, chan2016listen}. These breakthroughs have significantly advanced the field of artificial intelligence. Simultaneously, as supervised learning demonstrates significant performance improvements in real-world applications, researchers have gradually uncovered intriguing phenomena such as Neural Collapse~\cite{papyan2020prevalence}, linear mode connectivity~\cite{frankle2020linear}, and grokking~\cite{power2022grokking}. These phenomena have become the subject of growing research interest aimed at uncovering their underlying causes.

\input{figs/defination/defination}

Neural Collapse (NC)~\cite{papyan2020prevalence} is a compelling phenomenon observed during the training process of supervised learning. As training progresses, data representations within the same class become increasingly similar in the feature space, leading to reduced intra-class variability. Concurrently, data representations of different classes become more distinct, enhancing inter-class separability. In classification tasks, prolonged training often results in an alignment between the weights of the final fully connected layer and the corresponding class centroids. For each class, the centroid of its representations nearly coincides the weight vector of its corresponding classifier (i.e., the weights of the classification head).

Existing research on Neural Collapse has primarily focused on using similarity to represent the alignment between data representations and classification head weights. In this paper, we offer new theoretical insights into Neural Collapse through the lens of information theory. Calculating Shannon entropy requires first estimating the distribution of representations. To address this, we introduce matrix entropy as a precise analytical tool that does not require distribution estimation to describe the information content (\figureref{fig:defination_entropy}). First, we provide a theoretical analysis of the matrix entropy of data representations and classification head weights under Neural Collapse conditions. Observations during training show that the variation in matrix entropy aligns with our theoretical derivations. Furthermore, we identify an intriguing phenomenon: under varying temperature coefficients in the softmax function, the matrix entropy tends to decrease as the temperature increases. Through analyzing the representations of samples under different temperatures, we observe a consistent pattern: the matrix entropy decreases as clustering improves. This observation reveals a strong correlation between the matrix entropy of data representations and their clustering properties. Inspired by this, we propose a novel cross-modal alignment loss (CMA) to optimize the supervised fine-tuning of pre-trained models by aligning knowledge across different modalities. Experiments demonstrate that although matrix information entropy alone cannot fully determine the state of Neural Collapse, it serves as a valuable regularization term for optimizing knowledge alignment during supervised fine-tuning of cross-modal pre-trained models.

To further elucidate the intricate interplay of information in supervised learning, we introduce two novel metrics: the Matrix Mutual Information Ratio (MIR) and the Matrix Entropy Difference Ratio (HDR) (\figureref{fig:defination_hdr_mir}). Under Neural Collapse, the alignment between data representations and classification head weights results in identical matrix entropy values. Our theoretical analysis predicts the values of MIR and HDR under Neural Collapse conditions. Observations confirm that MIR and HDR between data representations and classification head weights closely approach their theoretical values, validating the effectiveness of these metrics. Additionally, our findings indicate that MIR and HDR can describe other phenomena in supervised learning, such as Linear Mode Connectivity and Grokking. Furthermore, information interplay metrics can be incorporated as additional loss terms to optimize the learning process (Fig. 1d). Experiments demonstrate that MIR and HDR not only assess Neural Collapse effectively but also improve model performance when used as regularization terms.

Our contributions are as follows:

1. We observe the relationship between matrix information entropy, sample representations, and classification head weights. Based on the properties of matrix information entropy, we propose a new cross-modal alignment (CMA) loss and use it to optimize the fine-tuning process of pre-trained models.

2. Experimental observations indicate that matrix information entropy alone cannot adequately describe Neural Collapse. Based on this, we propose two new metrics: Matrix Mutual Information Ratio (MIR) and Matrix Entropy Difference Ratio (HDR), for which we also deduce their theoretical values when Neural Collapse happens. Through rigorous experiments, we find that MIR and HDR are capable of explaining various phenomena, such as the standard training of supervised learning, linear mode connectivity, and grokking.

3. We integrate matrix mutual information and information entropy differences as a loss term in both supervised and semi-supervised learning. Experiments demonstrate that these information metrics can effectively improve model performance.

%% file: figs/defination/defination.tex
\begin{figure}[t]
    \begin{subfigure}{\linewidth}
        \centering
        \includegraphics[width=\linewidth]{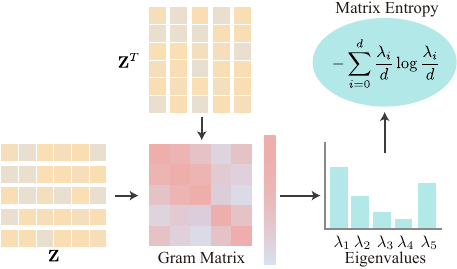}
        \caption{Matrix entropy.}
        \label{fig:defination_entropy}
    \end{subfigure}

    \begin{subfigure}{\linewidth}
        \centering
        \includegraphics[width=\linewidth]{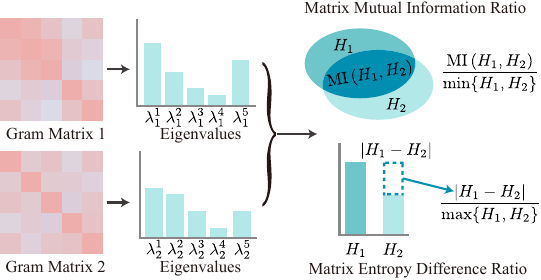}
        \caption{Matrix mutual information ratio, matrix entropy difference ratio.}
        \label{fig:defination_hdr_mir}
    \end{subfigure}
    
    \caption{The calculation of matrix entropy, matrix mutial information ratio and matrix entropy difference ratio.}
    \label{fig:defination}
\end{figure}

%% file: content/related.tex
\section{Related work}

\paragraph{Neural Network Training Phenomena}

Recent research has uncovered several intriguing phenomena that are crucial for understanding the behavior and learning dynamics of neural networks. Papyan~et~al.~\cite{papyan2020prevalence} observed that, during the final stages of deep neural network training, the feature vectors of the last layer tend to converge to their class centroids, which align with the weights of the corresponding classes in the final fully connected layer. This phenomenon is termed \textit{Neural Collapse}, and it is observed in both MSE and cross-entropy loss settings~\cite{han2021neural,zhou2022all}. Frankle~et~al.~\cite{frankle2020linear} found that models trained from the same initialization, even with variations in input data sequence and augmentation, converge to the same local area, a phenomenon called \textit{Linear Mode Connectivity}, which is influenced by architecture, training strategy, and dataset~\cite{altintacs2023disentangling}. Additionally, Power~et~al.~\cite{power2022grokking} discovered that prolonged training can transition models from memorization to inductive learning, a phenomenon known as \textit{Grokking}. Nanda~et~al.~\cite{nanda2022progress} explored the connections of Grokking on modulo addition tasks with trigonometric functions.

\paragraph{Information Theory}
Traditional information theory provides a foundational framework to understand the relationships between probability distributions and information~\cite{wang2021adaptive}. However, when dealing with high-dimensional and complex data structures, traditional information theory tools struggle to capture higher-order relationships. As an extension, matrix information theory expands the scope to analyze inter-matrix relationships, facilitating a deeper understanding of latent structures in data and addressing complex relationships in high-dimensional settings~\cite{bach2022information}. Recent studies have applied matrix mutual information to analyze neural networks. For example, Tan~et~al.~\cite{tan2023information} used matrix mutual information to study Siamese architecture in self-supervised learning, while Zhang~et~al.~\cite{zhang2023matrix} highlighted the connections between effective rank, matrix entropy, and equiangular tight frames.

\paragraph{Few-shot Fine-tuning}
Few-shot fine-tuning aims to fine-tune pretrained models using a small amount of data and apply them to downstream tasks. The data for fine-tuning and downstream tasks may come from the same or different distributions and categories. Methods like CoOp~\cite{CoOp} and CoCoOp~\cite{CoCoop} optimize prompt contexts to learn accurate category representations. MaPLe~\cite{maple} learns vision and language prompts to align multimodal representations. FD-Align~\cite{FDAlign} ensures out-of-distribution performance by aligning class-independent representations before and after fine-tuning. PromptSRC~\cite{promptSRC} introduces a self-regularization framework to optimize both task-specific and task-agnostic representations. These methods primarily focus on improving the accuracy of image and class representations.

\paragraph{Semi-supervised Learning}
Semi-supervised learning (SSL) seeks to improve model performance using a small number of labeled examples alongside a large amount of unlabeled data~\citep{sohn2020fixmatch, zhang2021flexmatch, chen2023softmatch, tan2023otmatch, wang2022freematch, tan2023seal, zhang2023relationmatch}. FixMatch~\cite{sohn2020fixmatch} integrates consistency regularization with pseudo-labeling. MixMatch~\cite{berthelot2019mixmatch} combines leading SSL methodologies, significantly reducing error rates while enhancing privacy. FlexMatch~\cite{zhang2021flexmatch} introduces curriculum pseudo-labeling, dynamically adapting to the model's learning status and proving effective in scenarios with limited labeled data. SoftMatch~\cite{chen2023softmatch} balances the quantity and quality of pseudo-labels, achieving significant performance improvements across diverse applications. FreeMatch~\cite{wang2022freematch} innovates by self-adaptively adjusting confidence thresholds and incorporating class fairness regularization, outperforming existing methods in scenarios with scarce labeled data. Accurately leveraging unlabeled data remains a pivotal challenge in the field of SSL.

%% file: content/prelim.tex
\section{Preliminaries}

\subsection{Supervised classification problem}
Given a labeled dataset $\{(\mathbf{x}_i , y_i )  \}^n_{i=1}$, where $y_i \in \{1, 2, \cdots, C   \}$ is the class label. In this paper, we focus on training an image classification model by combining of a deep neural network $h$ and a linear classifier. The linear classifier consists of a weight matrix $\mathbf{W} \in \mathbb{R}^{{C \times d}}$ and $\mathbf{b} \in \mathbb{R}^{{C \times 1}}$. Denote $\mathbf{W}^T = [w_1 \cdots w_C]$. The training process minimizes the cross-entropy loss:
$$
\mathcal{H}(p, q) = -\sum_{i=0}^n{p(x_i)\log q(x_i)},
$$ 
where $p$ is the true probability distribution, and $q$ is the predicted probability distribution.

\subsection{Matrix entropy and mutual information}

The following definitions of matrix entropy and matrix mutual information are taken from paper~\cite{skean2023dime}.

\begin{definition}[Matrix entropy]\label{ME definition}
Suppose a positive-definite matrix $\mathbf{K} \in \mathbb{R}^{d \times d}$ which ${\mathbf{K}(i, i)}=1$ ($1 \leq i \leq d$). The matrix entropy is defined as follows:
$$
\operatorname{H}\left(\mathbf{K}\right)=-\operatorname{tr}\left(\frac{1}{d} \mathbf{K} \log  \frac{1}{d} \mathbf{K} \right) = -\sum_{i=0}^d \frac{\lambda_i}{d} \log(\frac{\lambda_i}{d}).
$$
    
\end{definition}

\begin{definition}[Effective Rank\cite{roy2007effective}]\label{effective rank}
The effective rank of the matrix $\mathbf{A}$, donate $erank(\mathbf{A})$, is defined as
\begin{equation*}
    erank(\mathbf{A}) = exp(\operatorname{H}(p_1, p_2, \ldots, p_Q)),
\end{equation*}

where $p_i = \frac{\sigma_i}{\sum_{k=1}^n \sigma_k}, \{ \sigma_i|i=1,\ldots,n \}$ are the singular values of $\mathbf{A}$, and $\operatorname{H}(p_1, p_2, \ldots, p_Q)$ is the Shannon entropy given by $\operatorname{H}(p_1, p_2, \ldots, p_Q) = -\sum_{k=1}^Q p_k \log(p_k)$.

\end{definition}

\begin{definition}[Matrix mutual information]
The matrix mutual information is defined as follows:
$$
\operatorname{MI}\left(\mathbf{K}_1, \mathbf{K}_2\right) = \operatorname{H}\left(\mathbf{K}_1\right) + \operatorname{H}\left(\mathbf{K}_2\right) - 
 \operatorname{H}(\mathbf{K}_1 \odot \mathbf{K}_2) ,
$$
where $\odot$ is the Hardmard product.

\end{definition}

Based on the two definitions above, we can introduce the following concepts, which measure the normalized information interactions between matrices.

\begin{definition}[Matrix mutual information ratio (MIR)] \label{MIR}
The matrix mutual information ratio is defined as follows:
$$
\operatorname{MIR}\left(\mathbf{K}_{1}, \mathbf{K}_2 \right) = \frac{\operatorname{MI}\left(\mathbf{K}_1, \mathbf{K}_2\right)}{\min \{ \operatorname{H}(\mathbf{K}_1), \operatorname{H}(\mathbf{K}_2) \}}.
$$
    
\end{definition}

\begin{definition}[Matrix entropy difference ratio (HDR)] \label{HDR}
The matrix entropy difference ratio is defined as follows:
$$
\operatorname{HDR}\left(\mathbf{K}_{1}, \mathbf{K}_2 \right) = \frac{| \operatorname{H}(\mathbf{K}_1) - \operatorname{H}(\mathbf{K}_2) |}{\max \{ \operatorname{H}(\mathbf{K}_1), \operatorname{H}(\mathbf{K}_2) \}}.
$$
    
\end{definition}

%% file: content/method.tex
\section{Theoretic insights in supervised learning}

In this section, we first introduce some fundamental properties of Neural Collapse. Next, we describe the properties of matrix information entropy, matrix mutual information rate, and information entropy difference rate in the context of Neural Collapse.  Following this, we provide theoretical insights related to the matrix information entropy.

\subsection{Neural Collapse}

Neural Collapse (NC) is a remarkable phenomenon~\cite{papyan2020prevalence} observed during the terminal phase of the classification problem. We summarize the three most important NC conditions relevant to this paper as follows:

Denote $\mu_G =  \frac{\sum^n_{i=1}  h(\mathbf{x}_i)}{n}$ as the global mean and $\mu_c = \frac{\sum_{y_i=c}  h(\mathbf{x}_i)}{\# \{ y_i=c\}}$ as the class-wise mean. Then we define $\tilde{\mu}_c = \mu_c - \mu_G$.

(NC 1) $h(\mathbf{x}_i) =  \mu_{y_i}$  ($i=1,2,\cdots,n$). 

(NC 2) $\text{cos} (\tilde{\mu}_i , \tilde{\mu}_j) =  \frac{C}{C-1} \delta^i_j - \frac{1}{C-1}$, where $\text{cos}$ is the cosine similarity and $ \delta^i_j $ is Kronecker symbol.

(NC 3) $\frac{\mathbf{W}^T}{\| \mathbf{W}\|_F} =  \frac{\mathbf{M}}{\| \mathbf{M} \|_F}$, where $\mathbf{M}=[\tilde{\mu}_1 \cdots \tilde{\mu}_C]$.

In this paper, the matrices used in matrix information quantities are typically similarity (Gram) matrices. For clarity, we introduce a standard method for constructing a similarity (Gram) matrix as follows:

\begin{definition}[Construction of similarity (gram) matrix] \label{gram}
Given a set of representations $\mathbf{Z} = [\mathbf{z}_1 \cdots \mathbf{z}_N] \in \mathbb{R}^{d \times N}$. Denote the $l_2$ normalized feature $\hat{\mathbf{z}}_i = \frac{\mathbf{z}_i}{\| \mathbf{z}_i \|}, $ $ \hat{\mathbf{Z}} = [\hat{\mathbf{z}}_1 \cdots \hat{\mathbf{z}}_N] $. Then gram matrix is defined as $\mathbf{G}(\mathbf{Z}) = \hat{\mathbf{Z}}^T\hat{\mathbf{Z}}$. 
\end{definition}

\begin{theorem} \label{zero entropy}
    Given a set of representations $f = \smash{[ h(x_1), h(x_2), \ldots, h(x_n) ]}$, if $\smash{H(\mathbf{G}(f)) = 0}$, the similarities between any representations are $1$, i.e.,  all the representations are the same,  $\smash{h(x_1) = h(x_2) = \ldots = h(x_n)}$.
\end{theorem}

Note that Neural Collapse conditions impose structural information on the representation of the dataset, as well as on the weight matrix and class means. We provide the relationship between the matrix entropy of dataset's sample representation and the number of classes in \theoremref{entropy and class number}. In \theoremref{direct NC}, we reveal the structural information on the matrix mutual information ratio and matrix entropy difference ratio between the weight matrix and the class means.

\begin{theorem}[]\label{entropy and class number}
    Suppose Neural Collapse happens, $ erank(\mathbf{G}(\mathbf{M})) = C-1$. \textnormal{\cite{zhang2023matrix}} If the dataset is class-balance, for all representations $ f=[h(x_1), h(x_2), \ldots, h(x_n)] $ in datasets, $ {H}(\mathbf{G}(f)) = {H}(\mathbf{G}(\mathbf{W})) = {H}(\mathbf{G}(\mathbf{M})) = \log(C-1) $.
\end{theorem}

\begin{theorem} \label{direct NC}
Suppose Neural collapse happens. Then $\operatorname{HDR}(\mathbf{G}(\mathbf{W}^T), \mathbf{G}(\mathbf{M})) = 0$ and $\operatorname{MIR}(\mathbf{G}(\mathbf{W}^T), \mathbf{G}(\mathbf{M})) = \frac{1}{C-1} + \frac{(C-2)\log(C-2)}{(C-1)\log(C-1)}$. 
\end{theorem}

\begin{proof} \label{boundary appendix}

By (NC 3), we know that $\mathbf{W}^T = \frac{\| \mathbf{W}\|_F}{\| \mathbf{M} \|_F} \mathbf{M}$. Noting that $\frac{\| \mathbf{W}\|_F}{\| \mathbf{M} \|_F} > 0$, we know that $\frac{w_i}{\| w_i \|} = \frac{\tilde{\mu}_i}{ \| \tilde{\mu}_i |}$. It is then very clear that $\mathbf{G}(\mathbf{W}^T) =  \mathbf{G}(\mathbf{M})$. Therefore from \defineref{gram} and \defineref{HDR}, it is clear that $\operatorname{HDR}(\mathbf{G}(\mathbf{W}^T), \mathbf{G}(\mathbf{M})) = 0$. 

Define $\mathcal{E}(\alpha) = \begin{bmatrix}
    1 & \alpha & \cdots & \alpha \\   
    \alpha & 1 & \cdots & \alpha \\  
    \vdots & \vdots & \ddots & \vdots \\
    \alpha & \alpha & \cdots & 1 \\
\end{bmatrix}$. From (NC 2), we know that $\mathbf{G}(\mathbf{W}^T) =  \mathbf{G}(\mathbf{M}) = \mathcal{E}(\frac{-1}{C-1})$ and $\mathbf{G}(\mathbf{W}^T) \odot \mathbf{G}(\mathbf{M}) = \mathcal{E}(\frac{1}{(C-1)^2})$. Notice that $\mathcal{E}(\alpha) = (1-\alpha) \mathbf{I}_C + \alpha \mathbf{1}^T_C \mathbf{1}_C$, we can obtain its spectrum as $1-\alpha$ ($C-1$ times) and $1+ (C-1)\alpha$ ($1$ time). Therefore, we can obtain that $\operatorname{H}(\mathbf{G}(\mathbf{W}^T)) = \operatorname{H}(\mathbf{G}(\mathbf{M})) = \log (C-1)$. And $\operatorname{H}(\mathbf{G}(\mathbf{W}^T) \odot \mathbf{G}(\mathbf{M})) = - \frac{1}{C-1} \log \frac{1}{C-1} - (C-1) \frac{C-2}{(C-1)^2} \log \frac{C-2}{(C-1)^2} = \frac{1}{C-1} \log (C-1) - \frac{C-2}{C-1} \log (C-2) + \frac{2(C-2)}{C-1} \log (C-1) = (2-\frac{1}{C-1})\log (C-1) - \frac{C-2}{C-1} \log (C-2).$ Then then conclusion follows from \defineref{MIR}.
        
\end{proof}

The linear weight matrix $\mathbf{W}$ can be interpreted as prototype embedding for each class. Naturally, this motivates the consideration of mutual information and entropy difference between sample embeddings and label embeddings. We explore this further in \corollaryref{feature NC}.

\begin{corollary} \label{feature NC}
Suppose the dataset is class-balanced, $\mu_G =0$ and Neural collapse happens. Denote $\mathbf{Z}_1 = [h(\mathbf{x}_1) \cdots h(\mathbf{x}_n)] \in \mathbb{R}^{d \times n}$ and $\mathbf{Z}_2 = [w_{y_1} \cdots w_{y_n}] \in \mathbb{R}^{d \times n}$. Then $\operatorname{HDR}(\mathbf{Z}_1, \mathbf{Z}_2) = 0$ and $\operatorname{MIR}(\mathbf{Z}_1, \mathbf{Z}_2) = \frac{1}{C-1} + \frac{(C-2)\log(C-2)}{(C-1)\log(C-1)}$.
\end{corollary}

\textbf{Remark:} Observe that $ \frac{1}{C-1} + \frac{(C-2)\log(C-2)}{(C-1)\log(C-1)} \approx \frac{1}{C-1} + \frac{(C-2)\log(C-1)}{(C-1)\log(C-1)} = 1$. Additionally, note that MIR and HDR lie within the interval $[0,1]$. These properties highlight the significance of the quantities derived from \theoremref{direct NC} and \corollaryref{feature NC}, as HDR achieves its minimum possible value while MIR nearly attains its maximum possible value.

\subsection{Some theoretical insights for our proposed HDR}

Mutual information is a fundamental concept in information theory, providing an intuitive measure of the dependence between variables. Conversely, considering the difference in entropy may initially seem unconventional; however, we demonstrate that this quantity is intrinsically connected to comparing the approximation capabilities of different representations for the same target.

To facilitate theoretical analysis, this section focuses on the Mean Squared Error (MSE) regression loss.

The following \lemaref{approx} shows that the regression of two sets of representations $\mathbf{Z}_1$ and $\mathbf{Z}_2$ to the same target $\mathbf{Y}$ are closely related. And the two approximation errors are closely related to the regression error of $\mathbf{Z}_1$ to $\mathbf{Z}_2$.

\begin{lemma} \label{approx}
Suppose $\mathbf{W}^*_1, \mathbf{b}^*_1 = \argmin_{\mathbf{W}, \mathbf{b}}  \|\mathbf{Y} - (\mathbf{W} \mathbf{Z}_1 + \mathbf{b} \mathbf{1}_N )\|_F$. Then $\min_{\mathbf{W}, \mathbf{b}} \|\mathbf{Y} - (\mathbf{W} \mathbf{Z}_2 + \mathbf{b} \mathbf{1}_N )\|_F \leq \min_{\mathbf{W}, \mathbf{b}}  \|\mathbf{Y} - (\mathbf{W} \mathbf{Z}_1 + \mathbf{b} \mathbf{1}_N )\|_F + \| \mathbf{W}^*_1\|_F \min_{\mathbf{H}, \mathbf{\eta}}  \|\mathbf{Z}_1 - (\mathbf{H} \mathbf{Z}_2 + \mathbf{\eta} \mathbf{1}_N )\|_F$.   
\end{lemma}

\begin{proof}
Suppose $\mathbf{H}^*, \mathbf{\eta}^* = \argmin_{\mathbf{H}, \mathbf{\eta}}  \|\mathbf{Z}_1 - (\mathbf{H} \mathbf{Z}_2 + \mathbf{\eta} \mathbf{1}_N )\|_F$. Then $\min_{\mathbf{W}, \mathbf{b}}  \|\mathbf{Y} - (\mathbf{W} \mathbf{Z}_2 + \mathbf{b} \mathbf{1}_N )\|_F \leq  \|\mathbf{Y} - (\mathbf{W}^*_1\mathbf{H}^* \mathbf{Z}_2 + (\mathbf{b}^*_1+\mathbf{W}^*_1 \eta^*) \mathbf{1}_N )\|_F \leq \|\mathbf{Y} - (\mathbf{W}^*_1 \mathbf{Z}_1 + \mathbf{b}^*_1 \mathbf{1}_N ) \|_F + \|\mathbf{W}^*_1(\mathbf{Z}_1 - (\mathbf{H}^* \mathbf{Z}_2 + \mathbf{\eta}^* \mathbf{1}_N )) \|_F \leq  \|\mathbf{Y} - (\mathbf{W}^*_1\mathbf{H}^* \mathbf{Z}_2 + (\mathbf{b}^*_1+\mathbf{W}^*_1 \eta^*) \mathbf{1}_N )\|_F \leq \|\mathbf{Y} - (\mathbf{W}^*_1 \mathbf{Z}_1 + \mathbf{b}^*_1 \mathbf{1}_N ) \|_F + \|\mathbf{W}^*_1\|_F \|\mathbf{Z}_1 - (\mathbf{H}^* \mathbf{Z}_2 + \mathbf{\eta}^* \mathbf{1}_N ) \|_F$.     
\end{proof}

From \lemaref{approx}, we observe that the regression error of $\mathbf{Z}_1$ to $\mathbf{Z}_2$ plays a critical role in understanding the differences between representations. This relationship is further analyzed by bounding the regression error in terms of rank and singular values in \lemaref{rank and singuar}.

\begin{lemma} \label{rank and singuar}
Suppose $\mathbf{Z}_1 = [\mathbf{z}^{(1)}_1 \cdots \mathbf{z}^{(1)}_N] \in \mathbb{R}^{d{'} \times N}$ and $\mathbf{Z}_2 = [\mathbf{z}^{(2)}_1 \cdots \mathbf{z}^{(2)}_N] \in \mathbb{R}^{d \times N}$ and $\text{rank}(\mathbf{Z}_1) > \text{rank}(\mathbf{Z}_2)$. Denote the singular value of $\frac{\mathbf{Z}_1}{\sqrt{N}}$ as $\sigma_1 \geq \cdots \geq \sigma_{N}$. Then $\min_{\mathbf{H}, \mathbf{\eta}} \frac{1}{N} \|\mathbf{Z}_1 - (\mathbf{H} \mathbf{Z}_2 + \mathbf{\eta} \mathbf{1}_N )\|^2_F \geq \sum^{\text{rank}(\mathbf{Z}_1)}_{j= \text{rank}(\mathbf{Z}_2)+2} (\sigma_j)^2$.  
\end{lemma}

\begin{proof}
The proof idea is similar to \cite{garrido2023rankme}. Suppose $\mathbf{H}^*, \mathbf{\eta}^* = \argmin_{\mathbf{H}, \mathbf{\eta}} \frac{1}{N} \|\mathbf{Z}_1 - (\mathbf{H} \mathbf{Z}_2 + \mathbf{\eta} \mathbf{1}_N )\|^2_F$ and $r = \text{rank}(\mathbf{H}^* \mathbf{Z}_2 + \mathbf{\eta}^* \mathbf{1}_N )$.

Then from Eckart–Young–Mirsky theorem $\frac{1}{N} \|\mathbf{Z}_1 - (\mathbf{H}^* \mathbf{Z}_2 + \mathbf{\eta}^* \mathbf{1}_N )\|^2_F \geq \sum^{N}_{j= r+1} (\sigma^{(1)}_j)^2$. Note that $r \leq \text{rank}(\mathbf{Z}_2)+1$, and the singular values index bigger than the rank are $0$. The conclusion follows.
\end{proof}

The bound presented in \lemaref{rank and singuar} may not be immediately intuitive. Assuming the features are normalized, we derive the connection between the regression error and the ratio of ranks in \theoremref{rank ratio}.

\begin{theorem} \label{rank ratio}
Suppose $\| \mathbf{z}^{(1)}_j \|_2=1$, where ($1 \leq j \leq N$). Then lower bound of approximation error can be upper-bounded as follows:
$\sum^{\text{rank}(\mathbf{Z}_1)}_{j= \text{rank}(\mathbf{Z}_2)+2} (\sigma_j)^2 \leq \frac{\text{rank}(\mathbf{Z}_1)-\text{rank}(\mathbf{Z}_2)-1}{\text{rank}(\mathbf{Z}_1)} \leq 1-\frac{\text{rank}(\mathbf{Z}_2)}{\text{rank}(\mathbf{Z}_1)}$.
\end{theorem}

\begin{proof}
The proof is direct by noticing the summation of the square of singular values is $1$ and we have already ranked singular values by their indexes.    
\end{proof}

According to the work of Wei et~al.~\cite{wei2024large} and Zhang et~al.~\cite{zhang2023matrix}, $\exp{(\operatorname{H}(\mathbf{G}(\mathbf{Z}))}$ is an approximate of $\text{rank}(\mathbf{Z})$. Then we can see that $\frac{\text{rank}(\mathbf{Z}_2)}{\text{rank}(\mathbf{Z}_1)} \approx \exp{(\operatorname{H}(\mathbf{G}(\mathbf{Z}_2)) - \operatorname{H}(\mathbf{G}(\mathbf{Z}_1)))}$, making the entropy difference a surrogate bound for approximation error.

%% file: content/MESL.tex
\section{Matrix Entropy in Supervised Learning}

According to \theoremref{entropy and class number} and \theoremref{zero entropy}, matrix entropy effectively captures the structural information among samples, including aspects like similarity and clustering. This section primarily discusses the performance of matrix entropy in supervised learning entropy. Due to computational resource constraints, we approximate the dataset's matrix entropy using batch matrix entropy.

\subsection{Matrix information entropy during standard supervised learning}
\label{matrix entropy during training}

First, we examine the variation of matrix information entropy during the standard supervised learning process across different datasets and model architectures. Specifically, we train WideResNet-28-2 on CIFAR-10 and WideResNet-28-8 on CIFAR-100 using an SGD optimizer (momentum: 0.9, weight decay: $5e^{-4}$), an initial learning rate of 0.03 with cosine annealing, a batch size of 64, and a total of $2^{20}$ training iterations. Unless stated otherwise, all subsequent experiments follow this setup.

As illustrated in \figureref{fig:tp_entropy}, the matrix entropy of data representations is close to zero at the start of training. According to \theoremref{zero entropy}, this suggests that high similarity among data representations, meaning that initial representations cannot effectively distinguish samples from different classes. As training progresses, the matrix entropy of data representations increases, reflecting improved discrimination among samples and a simultaneous enhancement in the model’s accuracy.

Moreover, compared to the matrix entropy of data representations, the matrix information entropy of the classifier head weights is closer to the Neural Collapse state at the initial stage. This is because the randomly initialized classifier head weights differ significantly, resulting in an initial Gram matrix that is nearly an identity matrix. However, at this stage, the classifier head contains no class information, leading to very low classification performance. During the first few epochs of training, the matrix entropy of the classifier head weights decreases rapidly, indicating that the classifier head begins to effectively distinguish different classes. As training continues, the matrix entropy of the classifier head weights increases steadily, enhancing its ability to discriminate between different class information.

\input{figs/entropy/entropy}

According to Theorem \theoremref{entropy and class number}, the entropy of data representations and classifier head weights is related to the number of categories under the Neural Collapse state. However, by the end of training, the entropy of data representations and classification head weights on CIFAR-10 and CIFAR-100 does not reach the Neural Collapse state (i.e., the entropy of data representations and classifier head weights for CIFAR-10 is $\ln 9$ and for CIFAR-100, it is $\ln 99$). On CIFAR-10, although the data representations approach the Neural Collapse state during training, their entropy continues to increase because the classification head weights have not yet reached the Neural Collapse state. On CIFAR-100, neither the entropy of data representations nor the classification head weights reaches the Neural Collapse state by the end of training.

In summary, while the theoretical values of information entropy for data representations and classifier head weights under the Neural Collapse can be derived, the inconsistency in training progress between the feature extractor and the classifier head means that relying solely on the entropy of data representations or classifier weights is insufficient to determine whether the model has reached the Neural Collapse.

\subsection{Matrix entropy in Softmax}
\label{matrix entropy in Softmax}

Softmax is a widely used function in machine learning to transform representations into probability distributions, with the temperature coefficient playing a critical role in controlling the smoothness of this distribution. \figureref{fig:t_feaie} illustrates the accuracy and information entropy of sample representations for models trained with varying temperature coefficients. While accuracy shows minimal variation across different temperatures, the information entropy of the sample representation matrix decreases significantly as the temperature increases. According to \theoremref{zero entropy}, lower representation information entropy implies higher similarity among representations, leading to improved clustering performance.

\input{figs/softmaxT/FEAIE/FEAIE}

\input{figs/softmaxT/cluster/cluster}

To quantitatively evaluate the clustering effectiveness of representations, we utilize the Silhouette Coefficient~\cite{rousseeuw1987silhouettes} and Davies-Bouldin Index~\cite{davies1979cluster}  as metrics. The Silhouette Coefficient measures how well a sample aligns with its own class center compared to other classes: $S(i) = \frac{b(i)-a(i)}{\max(a(i), b(i))}$, where $a(i)$ is the average distance between a sample and all other points in the same cluster, and $b (i)$ is the average distance between a sample and all points in the nearest neighboring cluster. The Davies-Bouldin Index assesses clustering compactness and separation through the ratio of within-cluster scatter to between-cluster separation: $R_{ij}=\frac{S_i + S_j}{M_{ij}}$, where $S_i$ represents the average distance between points in a cluster and its centroid, and $M_{ij}$ is the distance between the centroids of clusters $i$ and $j$. As depicted in \figureref{fig:t_cluster}, representations extracted by models trained with higher temperature coefficients exhibit higher Silhouette Coefficients and lower Davies-Bouldin Index values. Comparing this with \figureref{fig:t_feaie}, it becomes evident that lower information entropy correlates with superior clustering performance. Additionally, we visualize the features extracted by models trained with temperature coefficients of 1 and 10. As shown in \figureref{fig:t_tsne100}, the features extracted by the model with a temperature coefficient of 10 are more compact within the same class and display greater inter-class separation compared to the model trained with a temperature coefficient of 1.

\input{figs/softmaxT/tSNE/cifar100}

%% file: figs/entropy/entropy.tex
\begin{figure}[t]
    \centering
    \begin{subfigure}{0.49\linewidth}
        \centering
        \includegraphics[width=\linewidth]{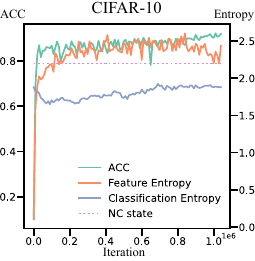}
        \caption{CIFAR-10}
        \label{fig:entropy_10}
    \end{subfigure}
    \begin{subfigure}{0.49\linewidth}
        \centering
        \includegraphics[width=\linewidth]{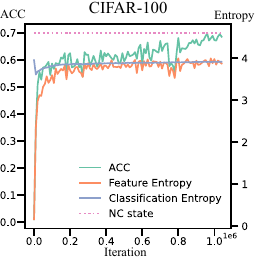}
        \caption{CIFAR-100}
        \label{fig:entropy_100}
    \end{subfigure}
    \caption{Variations in model accuracy and the matrix information entropy of data representations and classifier weights during the training process on CIFAR-10 and CIFAR-100.}
    \label{fig:tp_entropy}
\end{figure}

%% file: figs/softmaxT/FEAIE/FEAIE.tex
\begin{figure}[t]
    \centering
    \begin{subfigure}{0.49\linewidth}
        \centering
        \includegraphics[width=\linewidth]{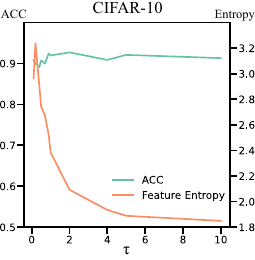}
        \caption{CIFAR-10}
    \end{subfigure}
    \begin{subfigure}{0.49\linewidth}
        \centering
        \includegraphics[width=\linewidth]{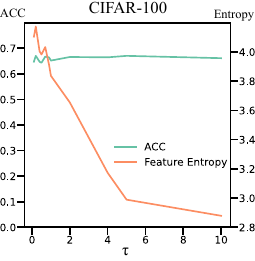}
        \caption{CIFAR-100}
    \end{subfigure}
    \caption{Relationship between accuracy, matrix entropy of data representations, and softmax temperature.}
     \label{fig:t_feaie}
\end{figure}

%% file: figs/softmaxT/cluster/cluster.tex
\begin{figure}[b]
    \centering
    \begin{subfigure}{0.49\linewidth}
        \centering
        \includegraphics[width=\linewidth]{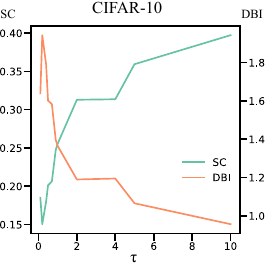}
        \caption{CIFAR-10}
    \end{subfigure}
    \begin{subfigure}{0.49\linewidth}
        \centering
        \includegraphics[width=\linewidth]{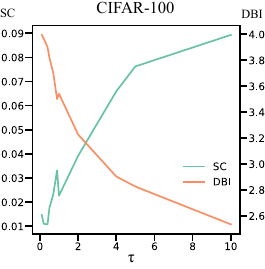}
        \caption{CIFAR-100}
    \end{subfigure}
    \caption{The SC (Silhouette Coefficient) and DBI (Davies-Bouldin Index) of representation extracted by models trained with different temperature coefficients. }
    \label{fig:t_cluster}
\end{figure}

%% file: figs/softmaxT/tSNE/cifar100.tex
\begin{figure}[t]
    \centering
    \begin{subfigure}{0.49\linewidth}
        \centering
        \includegraphics[width=\linewidth]{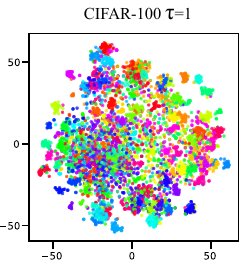}
        \caption{$\tau: 1$}
    \end{subfigure}
    \begin{subfigure}{0.49\linewidth}
        \centering
        \includegraphics[width=\linewidth]{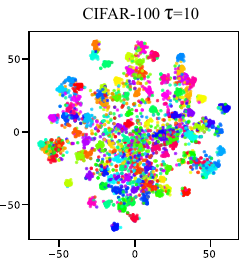}
        \caption{$\tau: 10$}
    \end{subfigure}
    \caption{Train models on CIFAR-100 with temperature coefficients set to 1 and 10, respectively, and visualize the test set features using t-SNE.}
    \label{fig:t_tsne100}
\end{figure}

%% file: content/IRSL.tex
\section{Information interplay in supervised learning}

According to \secref{matrix entropy during training}, matrix entropy can effectively describe the sample representations and the training state of the fully connected layer during training. However, it cannot accurately represent the training state of the entire model. To address this issue, inspired by matrix information theory and Neural Collapse theory, we focus on the consistency between sample representations and class classification heads. We determine the relationships among samples by constructing a similarity matrix of the dataset sample representations. According to NC1 and NC3, the similarity matrix between samples approximates the similarity matrix of the corresponding class centers, which also represents the similarity matrix of the corresponding weights in the fully connected layer. Therefore, under Neural Collapse, the similarity relationships among samples are equivalent to the similarity relationships of the corresponding category weights in the fully connected layer. Our analysis, grounded in matrix information theory, primarily examines the relationship between the representations of samples and the weights in the fully connected layer. Due to computational resource constraints, we approximate the dataset's matrix entropy using batch matrix entropy.

\input{figs/MIR_HDR/MIR_HDR}

\subsection{Information interplay during standard supervised learning process }

According to Neural Collapse, during the terminal stages of training, sample features align with the weights of the fully connected layer. \theoremref{direct NC} indicates that during the training process, MIR increases to its theoretical upper limit, while HDR decreases to zero. We plot the model's accuracy on the test set during training, along with the MIR and HDR between data representations and the corresponding classification heads. As shown in \figureref{fig:tp_MIR}, on CIFAR-10 and CIFAR-100, the accuracy and MIR exhibit almost identical variation trends. In most cases, both accuracy and MIR increase or decrease simultaneously, with MIR consistently showing an upward trend toward its theoretical maximum value. During training, accuracy and HDR typically show opposite trends, with HDR continually decreasing, even nearing its theoretical minimum value of zero on CIFAR-100. In summary, MIR and HDR effectively describe the training process towards Neural Collapse.

\input{figs/linear_connectivity/MIRHDR/MIRHDR}

\input{figs/linear_connectivity/lr/lr}

\subsection{{Information interplay in linear mode connectivity}}

Linear mode connectivity \cite{frankle2020linear} suggests that under specific datasets and experimental setups, models initialized with the same parameters will be optimized near the same local optimal basin, even if the order of training data and data augmentation differs. We investigate the behaviors of MIR and HDR under the setting of linear mode connectivity. We initialize models with the same random parameters and train them using different data sequences and random augmentations. Subsequently, we linearly interpolate these two checkpoints to obtain a new model $h=(1-\omega)\cdot h_1 + \omega \cdot h_2$, where $h_1$ and $h_2$ are the two checkpoints, and $\omega$ is the interpolation weight. We then test these models on the test set for accuracy, MIR, and HDR.

We conduct experiments on CIFAR-10 and CIFAR-100. As shown in \figureref{fig:lc_HDRMIR_10} and \figureref{fig:lc_HDRMIR_100}, on CIFAR-100, the performance of models obtained along the interpolation line is consistent with linear mode connectivity. At this point, MIR and HDR remain nearly unchanged. However, on CIFAR-10, the models do not exhibit linear mode connectivity. When the interpolation weight is between 0.4 and 0.6, the performance of the interpolated models drops to that of random guessing. Surprisingly, during this period, MIR shows an additional upward trend. Moreover, when the interpolation weight is close to 0 or 1, despite a slight decrease in performance, HDR also decreases. Although difficult to explain, this anomaly shows that HDR and MIR differ from accuracy, offering an intriguing avenue for further exploration.

Alt{\i}nta{\c{s}} et al.~\cite{altintacs2023disentangling} point out that linear mode connectivity is related to the experimental configuration. Therefore, we posit that the performance decline of the interpolated model on CIFAR-10 is associated with an excessively high learning rate. During training, models navigate the loss landscapes in search of minima, and two models with linear mode connectivity are optimized near the same local optimum. When the learning rate is too high, different training sample orderings and data augmentations direct model optimization towards distinct regions within the loss landscape. We experiment with different learning rates on CIFAR-10 to test their linear mode connectivity. It is observed that as the learning rate decreases, fluctuations in accuracy, MIR, and HDR also reduce. When the learning rate is lowered to $3e^{-4}$, the model demonstrates linear mode connectivity. This suggests that HDR and MIR can effectively describe linear mode connectivity when it exists.

\input{figs/grokking/grokking}

\subsection{{Information interplay in Grokking}}

In supervised learning, training models on certain datasets can result in an anomalous situation. Initially, models quickly learn the patterns of the training set, but at this point, their performance on the test set remains very poor. As training continues, the models gradually learn representations that generalize to the test set, a phenomenon referred to as Grokking~\cite{nanda2022progress}. We explore the information interplay during Grokking. Following \cite{nanda2022progress}, we train a transformer to learn modular addition $ c \equiv (a + b)\pmod{p} $, where $p$ is 113. The model input is ``$a~b = $'', where $a$ and $b$ are encoded into $p$-dimensional one-hot vectors, and ``$=$'' signifies the output value $c$. Our model employs a single-layer ReLU transformer with a token encoding dimension of 128, four attention heads each of dimension 32, and an MLP with a hidden layer of dimension 512. We train the model using full-batch gradient descent with a learning rate of 0.001 and an AdamW optimizer with a weight decay parameter of 1. We use 30\% of all possible inputs ($113 \times 113$ pairs) as training data and test performance on the remaining 70\%.

As shown in \figureref{fig:grokking}, we plot the accuracy of both the training and test sets during the Grokking process, as well as the variation in MIR and HDR between the representation and the fully connected layer. In the early stages of training, the model quickly fits the training data, achieving 100\% accuracy on the training set. However, at this point, test set performance is nearly equivalent to random guessing. As training continues, the model gradually exhibits generalization capability on the test set, ultimately achieving 100\% accuracy, a hallmark of Grokking. \figureref{fig:grokking} also reveals a clear two-phase variation in both MIR and HDR between data representation and the fully connected layer. Initially, similar to fully supervised learning, MIR increases while HDR decreases. However, as training proceeds, MIR begins to decrease, and HDR starts to increase, indicating the model is seeking new optimal points. After the model achieves Grokking, MIR reaches its lowest point, and HDR rapidly declines from its highest point. These experiments demonstrate that HDR and MIR exhibit distinct phenomena in two stages, suggesting that information metrics can describe the Grokking phenomenon, providing a basis for further research.

%% file: figs/MIR_HDR/MIR_HDR.tex
\begin{figure}[t]
    \centering
    \begin{subfigure}{0.49\linewidth}
        \centering
        \includegraphics[width=\linewidth]{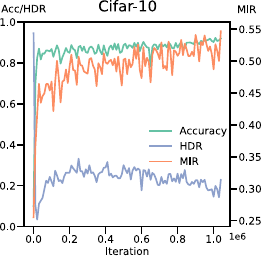}
        \caption{CIFAR-10}
        \label{fig:tp_MIR_10}
    \end{subfigure}
    \begin{subfigure}{0.49\linewidth}
        \centering
        \includegraphics[width=\linewidth]{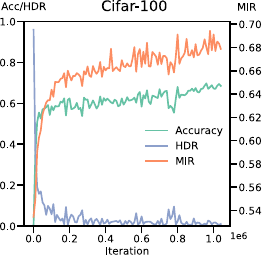}
        \caption{CIFAR-100}
        \label{fig:tp_MIR_100}
    \end{subfigure}
    \caption{Changes in model accuracy, matrix entropy of data representations, and classification head weights during training on CIFAR-10 and CIFAR-100}
    \label{fig:tp_MIR}
\end{figure}

%% file: figs/linear_connectivity/MIRHDR/MIRHDR.tex
\begin{figure}[b]
    \centering
    \begin{subfigure}{0.49\linewidth}
        \centering
        \includegraphics[width=\linewidth]{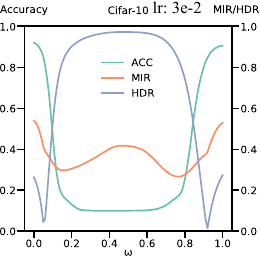}
        \caption{CIFAR-10}
        \label{fig:lc_HDRMIR_10}
    \end{subfigure}
    \begin{subfigure}{0.49\linewidth}
        \centering
        \includegraphics[width=\linewidth]{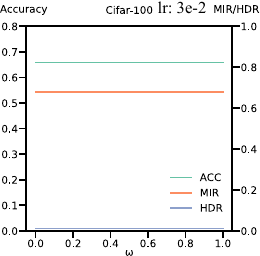}
        \caption{CIFAR-100}
        \label{fig:lc_HDRMIR_100}
    \end{subfigure}
    \caption{Train two models on CIFAR-100 and CIFAR-10 with different initializations and a learning rate of $3e^{-2}$. Interpolate between the models to create a new one and analyze the relationship between its accuracy, HDR, MIR, and the interpolation weights.}
    \label{fig:lc_HDRMIR}
\end{figure}

%% file: figs/linear_connectivity/lr/lr.tex
\begin{figure}[t]
    \centering
    \begin{subfigure}{0.49\linewidth}
        \centering
        \includegraphics[width=\linewidth]{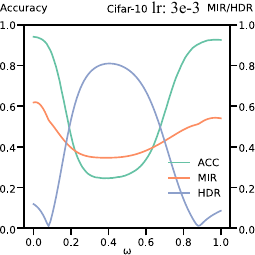}
        \caption{lr: $3e^{-3}$}
        \label{fig:lc_lr3}
    \end{subfigure}
    \begin{subfigure}{0.49\linewidth}
        \centering
        \includegraphics[width=\linewidth]{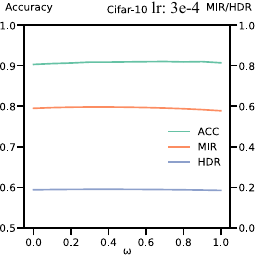}
        \caption{lr: $3e^{-4}$}
        \label{fig:lc_lr4}
    \end{subfigure}
    \caption{Train models on CIFAR-10 using different learning rates $3e^{-3}$, $3e^{-4}$ and analyze the impact of learning rates on model interpolation.}
    \label{fig:lc_lr}
\end{figure}

%% file: figs/grokking/grokking.tex
\begin{figure}[b]
    \centering
    \includegraphics[width=0.49\linewidth, trim=0 0 0 0, clip]{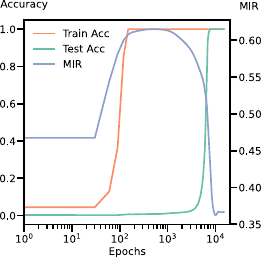}
    \includegraphics[width=0.49\linewidth, trim=0 0 0 0, clip]{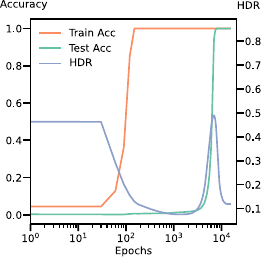}
    \caption{The relationship among Accuracy, MIR and HDR during Grokking.}
    \label{fig:grokking}
\end{figure}

%% file: content/FSFT.tex
\input{figs/CMA/CMA}

\section{Improving Cross-Modal Alignment with Matrix Entropy}

In cross-modal recognition tasks, aligning features from different modalities is crucial. As discussed in \secref{matrix entropy in Softmax}, reducing matrix entropy can effectively enhance the clustering performance of features. This section follows CoOp, using a small number of samples to fine-tune CLIP. Our method builds upon CoOp by exploring the influence of matrix entropy to enhance the model's cross-modal capabilities.

\subsection{Pipeline of cross-modal few-shot fine-tuning}

In cross-modal few-shot fine-tuning, we use a few-shot dataset \(\mathcal{D} \subset \mathcal{X} \times \mathcal{Y}\), where each image \( x \in \mathcal{X} \) is paired with its corresponding label name \( y \in \mathcal{Y} \). This dataset \(\mathcal{D}\) is employed to fine-tune the cross-modal pre-trained model, which consists of an image encoder \( f_\theta \) and a corresponding text encoder \( g_\theta \). The weights of the classifier are initialized using the text encoder as \(\{\mathbf{W}_i\}_{i=1}^C = g_\theta([P, y_i])\), where \(P\) represents the prompt tokens and \( C \) is the number of classes. The feature for each image is then calculated as $f_\theta(x)$, and the prediction probability is given by:
\begin{equation}
    p(y=i|x)=\frac{\exp(\cos(\mathbf{W}_i, f_\theta(x)))}{\sum_{j=1}^C \exp(\cos(\mathbf{W}_i, f_\theta(x))) }.
\end{equation}
The model is optimized using the cross-entropy loss:
\begin{equation}
    \mathcal{L}_{ce} = \frac{1}{B} \sum_{i=1}^B \mathcal{H}(y_i, p(x_i)),
\end{equation}
where $B$ represents the batch size, \(\mathcal{H}\) denotes the cross-entropy loss, and \(p(x_i)\) refers to the model's output probability for \( x_i \). The essence of cross-modal fine-tuning is to align features from different modalities.

\input{Alg/CMA}

\subsection{Aligning cross-model feature with matrix entropy}

To better align features across different modalities, we utilize features from different modalities to construct a cross-modal Gram matrix and compute its information entropy. As shown in \secref{matrix entropy in Softmax}, well-clustered features exhibit lower matrix entropy. Therefore, as illustrated in \figureref{fig:CMA}, we improve cross-modal feature alignment by minimizing the entropy of the cross-modal Gram matrix. The implementation details are provided in \algref{alg:cma}.

The final optimization objective is 
\begin{equation}
    \mathcal{L} = (1-\lambda) \cdot \mathcal{L}_{ce} + \lambda \cdot \mathcal{L}_{cma},
\end{equation}
where $\lambda$ is the weight of cross-modal alignment loss.

\subsection{Performance on few-shot fine-tuning}

\input{figs/FSFT_ID/FSFT_ID}

\input{figs/base2new/base2new}

Following CoOp, we use the open-source ResNet-50 as the backbone for CLIP and evaluate our method on 11 diverse datasets, including ImageNet~\cite{russakovsky2015imagenet}, StanfordCars~\cite{krause20133d}, UCF101~\cite{soomro2012ucf101}, Caltech101~\cite{fei2004learning}, OxfordFlowers~\cite{nilsback2008automated}, SUN397~\cite{xiao2010sun}, DTD~\cite{cimpoi2014describing}, EuroSAT~\cite{helber2019eurosat}, FGVCAircraft~\cite{maji2013fine}, OxfordPets~\cite{parkhi2012cats}, and Food101~\cite{bossard2014food}. These datasets encompass a variety of visual recognition tasks, such as generic object classification, fine-grained classification, action recognition, scene understanding, and texture analysis. We primarily compare our method with VNE~\cite{kim2023vne}, another approach that leverages entropy to optimize features.

As shown in \figureref{fig:FSFT_ID}, we compare the performance of three scenarios: vanilla CoOp, CoOp optimized with VNE, and CoOp optimized with CMA. The results demonstrate that CMA outperforms both CoOp and VNE in terms of average performance across all 11 datasets. In most cases, CMA significantly enhances CoOp's performance and clearly surpasses VNE, showcasing its ability to align cross-modal features more effectively.

Specifically, while VNE improves clustering effects across categories and modalities by optimizing the Gram matrix entropy, it suffers from two key drawbacks: 1) globally reducing matrix entropy may cause different categories to collapse into the same cluster, and 2) significant feature differences between modalities can lead to misaligned category features. In contrast, CMA optimizes the matrix entropy of features within the same category but across different modalities, effectively mitigating these issues.

To evaluate the impact of CMA and VNE on model generalization, we follow CoCoOp's base-to-new evaluation protocol. The datasets are divided into base classes and novel classes. The model is trained on the base classes (16-shots) and tested on the novel classes. As shown in \figureref{fig:FSFT_B2N}, CMA consistently outperforms CoCoOp and VNE in the base-to-new setting, indicating that it preserves the generalization ability of the pre-trained model.

\input{figs/similarity/similarity}

\input{table/base2new}

We also examine the effect of different $\lambda$ values on model performance. As shown in \tableref{fig:b2nlambda}, performance improves across most datasets for different $\lambda$, demonstrating the stability of CMA. For datasets with fewer categories, such as DTD and EuroSAT, higher $\lambda$ values yield better results, and performance using only the CMA loss surpasses that using only the cross-entropy loss. For datasets with larger category counts, such as ImageNet, the optimal $\lambda$ is 0.1, with higher values causing performance degradation. For datasets with a moderate number of categories, the optimal $\lambda$ typically falls between 0.4 and 0.6. We attribute this behavior to the fact that for datasets with fewer categories, aligning features within the same category across modalities does not significantly affect features from other categories, thereby improving performance. However, for datasets with a larger number of categories, excessive alignment may interfere with the model's inherent modality-alignment capabilities, resulting in performance drops.

To demonstrate CMA's effectiveness in aligning representations across modalities, we measured the similarity of features within the same category across different modalities and the overall similarity between data representations across modalities. As illustrated in \figureref{fig:CMA_sim}, CMA achieves superior alignment of cross-modal representations compared to CoCoOp and VNE.

%% file: figs/CMA/CMA.tex
\begin{figure*}[t]
  \centering
  \includegraphics[width=\textwidth]{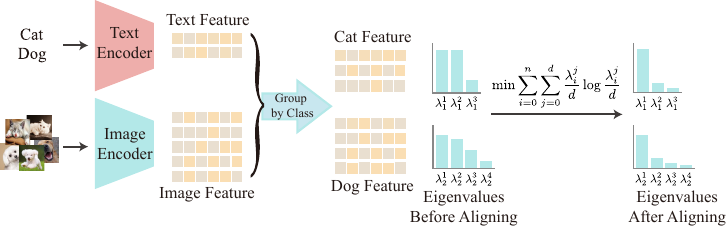}
  \caption{ Align domain features with cross-modal alignment (CMA) loss. First, extract the image features and text features. Then, group features by class. Finally, calculate the matrix entropy of each class and minimize the sum of matrix entropy.}
  \label{fig:CMA}
\end{figure*}

%% file: Alg/CMA.tex
\begin{algorithm}[t]
\caption{Cross Modal Alignment Loss}
\label{alg:cma}
\KwIn{Features of a batch of data and corresponding label \( F = \left\{(f_\theta(x_i), y_i)\right\}_{i=1}^N \). The weight of the classifier $ \{\mathbf{W}_i\}_{i=1}^C $.}
\KwOut{The cross modal alignment loss \(\mathcal{L}_{cma} \) for a batch of data.}
\KwInit{ \( \mathcal{L}_{cma} = 0 \), $List \leftarrow [~[~], \ldots, [~]~]$ \tcp{\( 1 \times C \)  empty list.}}
\For{ each $(f_\theta(x_i), y_i) $ in $ F $ }{
    Append $ f_\theta(x_i) $ to $ List[y_i] $;
}

\For{$i = 1$ \KwTo $C$}{
    Append $ \mathbf{W}_i $ to $ List[i] $;
}

\For{$i = 1$ \KwTo $C$}{
    \If {LENGTH(List[i]) $>$ 1}{
        $ \mathcal{L}_{cma} \leftarrow \mathcal{L}_{cma} + \frac{H(\mathbf{G}(List[i]))}{LENGTH(List[i])} $
    }
}

\Return $ \mathcal{L}_{cma} $

\end{algorithm}

%% file: figs/FSFT_ID/FSFT_ID.tex
\begin{figure*}[t]
  \centering
  \includegraphics[width=\textwidth]{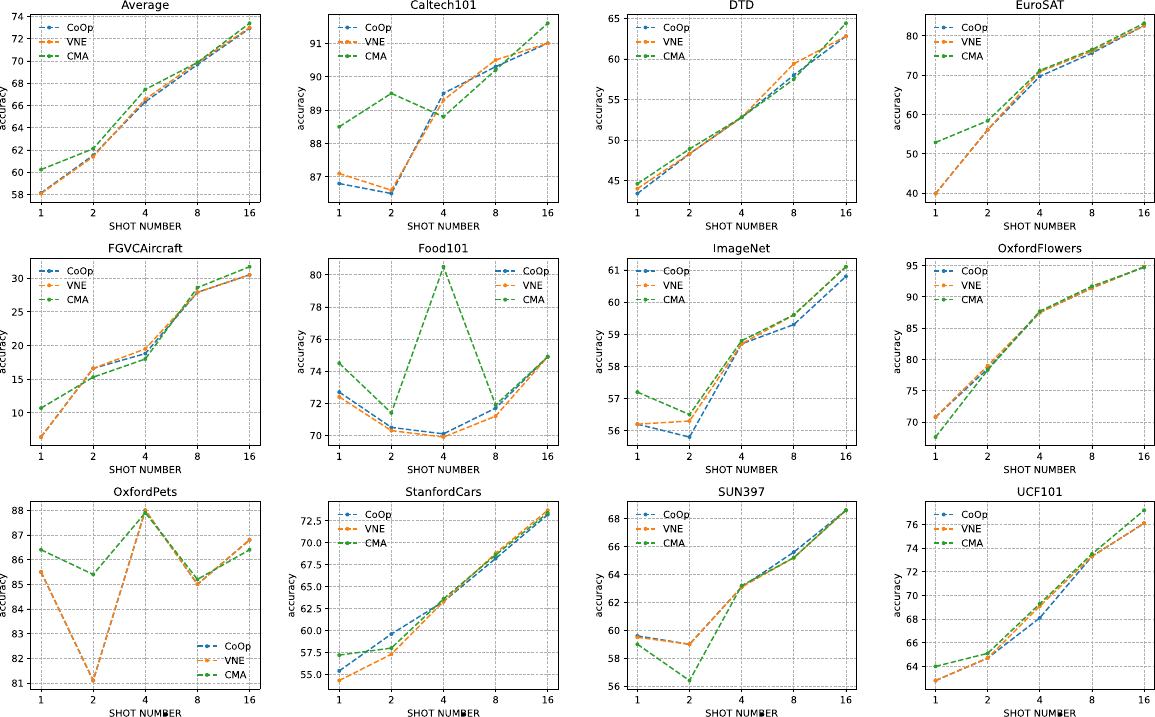}
  \caption{Performance comparison of different methods on 11 datasets.}
  \label{fig:FSFT_ID}
\end{figure*}

%% file: figs/base2new/base2new.tex
\begin{figure}[b]
  \centering
  \includegraphics[width=0.8\linewidth]{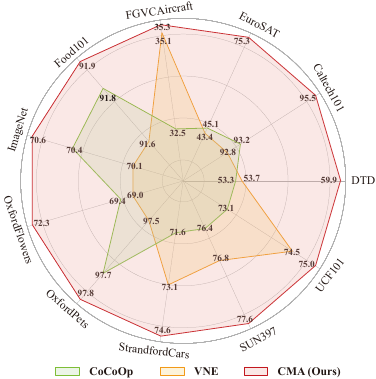}
  \caption{Base-to-new performance on 11 datasets.}
  \label{fig:FSFT_B2N}
\end{figure}

%% file: figs/similarity/similarity.tex
\begin{figure}[t]
  \centering
  \includegraphics[width=0.5\linewidth]{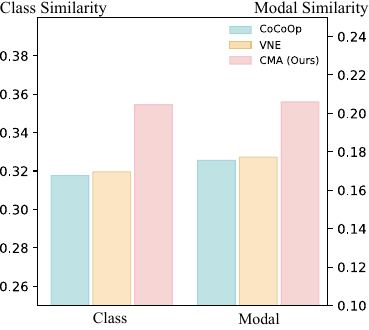}
  \caption{Cross modal similarity.}
  \label{fig:CMA_sim}
\end{figure}

%% file: table/base2new.tex
\begin{table}[h]
\centering
\caption{Performance of novel classes under different $\lambda$. When $\lambda=0$, it indicates the performance of vanilla CoCoOp.}
\resizebox{\columnwidth}{!}{%
\begin{tabular}{lccccccccccc}
\toprule
Dataset         & 0          & 0.1  & 0.2  & 0.3  & 0.4  & 0.5  & 0.6  & 0.7  & 0.8  & 0.9  & 1    \\ \midrule
Caltech101~\cite{fei2004learning}      & 93.2       & 92.5 & 92.8 & 93.2 & 93.0 & 93.2 & \textbf{95.5} & 93.3 & 93.8 & 92.5 & 88.5 \\
DTD~\cite{cimpoi2014describing}             & 52.3       & 56.3 & 56.3 & 54.5 & 59.8 & 58.9 & 57.6 & 54.0 & 52.2 & \textbf{59.9} & 56.5 \\
EuroSAT~\cite{helber2019eurosat}         & 45.1       & 47.2 & 47.8 & 44.2 & 55.9 & 53.1 & 59.6 & 59.6 & 66.1 & \textbf{75.3} & 62.8 \\
FGVCAircraft~\cite{maji2013fine}  & 32.5       & 33.7 & 32.9 & 34.0 & 33.5 & \textbf{35.3} & 33.4 & 30.7 & 30.2 & 31.3 & 23.5 \\
Food101~\cite{bossard2014food}         & 91.8       & 91.4 & 91.8 & 91.6 & 91.6 & 91.6 & \textbf{91.9} & 91.8 & 91.6 & 87.8 & 80.4 \\
ImageNet~\cite{russakovsky2015imagenet}        & 70.4       & \textbf{70.6} & 70.4 & 70.1 & 70.3 & 70.0 & 69.7 & 69.3 & 68.7 & 61.1 & 46.7 \\
OxfordFlowers~\cite{nilsback2008automated} & 69.4       & 71.2 & 72.3 & 68.9 & 61.7 & \textbf{71.9} & 70.2 & 69.1 & 70.5 & 65.5 & 58.4 \\
OxfordPets~\cite{parkhi2012cats}    & 97.7       & \textbf{97.8} & 97.7 & 97.6 & 97.4 & \textbf{97.8} & \textbf{97.8} & 96.7 & 97.7 & 97.3 & 84.8 \\
StanfordCars~\cite{krause20133d}  & 71.6       & 72.9 & 74.3 & 73.1 & \textbf{74.6} & \textbf{74.6} & 73.3 & 73.5 & 74.0 & 66.7 & 57.5 \\
SUN397~\cite{xiao2010sun}          & 76.4       & 76.8 & 77.1 & 76.4 & 77.2 & 77.3 & 77.5 & 77.0 & \textbf{77.6}& 72.1 & 61.5 \\
UCF101~\cite{soomro2012ucf101}          & 73.1       & 72.8 & 72.6 & 72.9 & 74.5 & 73.4 & \textbf{75.0}& 74.8 & 69.9 & 72.0 & 58.2 \\ \bottomrule
\end{tabular}
}
\label{fig:b2nlambda}
\end{table}

%% file: content/SSL.tex
\section{{Improving supervised and semi-supervised learning with information interplay}}

\subsection{Pipeline of supervised and semi-supervised learning}

In this section, we detail the application of matrix information entropy in supervised and semi-supervised learning. For supervised learning, a neural network $h$ and classifier $\mathbf{W} \in \mathbb{R}^{C\times d}$ are trained on the dataset $\mathcal{D}_L={(x_i, \tilde{y}i)}_{i=0}^{N_L}$, which contains $N_L$ samples. Here, $h$ extracts data features $f \in \mathbb{R}^D$, while $\mathbf{W}$ classifies the extracted features. The model is optimized using the cross-entropy loss:
\begin{equation*}
    \mathcal{L}_s=\frac{1}{B}\sum_{i=1}^B\mathcal{H}(y_i, p(\omega(x_i))),
\end{equation*}
where $B$ denotes the batch size, $\mathcal{H}$ represents the cross-entropy loss, $p(\cdot)$ is the model's output probability for a sample, and $\omega$ refers to random data augmentation.

In semi-supervised learning, an additional unlabeled dataset $\mathcal{D}_U=\{u_i\}_{i=0}^{N_U}$, containing $N_U$ unlabeled samples, is used to optimize the model further. For processing unlabeled data, we follow the approach outlined in FreeMatch \cite{wang2022freematch}, which involves generating pseudo-labels through weak data augmentation and selecting samples based on a probability threshold. Strongly augmented data features are then used to compute the cross-entropy loss with pseudo-labels. The training objective for unlabeled data is:
\begin{equation*}
    \mathcal{L}_u=\frac{1}{\mu B}\sum_{i=1}^{\mu B}\mathbb{I}\left(max(q_i) > \tau\right)\cdot \mathcal{H}\left(\hat{q_i},Q_i\right),
\end{equation*}
where $q_i$ and $Q_i$ correspond to $p(y|\omega(u_i))$ and $p(y|\Omega(u_i))$, respectively. $\hat{q_i}$ represents one-hot pseudo-labels generated from $q_i$, and $\mathbb{I}(\cdot > \tau)$ is an indicator function for values exceeding the threshold $\tau$. $\omega$ and $\Omega$ distinguish weak and strong data augmentations, respectively.

FreeMatch also incorporates a fairness objective to ensure uniform frequency prediction across classes:
\begin{equation*}
    \mathcal{L}_f=-H\left(\text{SumNorm}\left(\frac{p_1}{hist_1}\right), \text{SumNorm}\left(\frac{p_2}{hist_2}\right)\right),
\end{equation*}
where $\text{SumNorm}(\cdot)=(\cdot)/\sum(\cdot)$. $p_1$ and $p_2$ represent the average predictions under weak and strong augmentations, while $hist_1$ and $hist_2$ are the corresponding histogram distributions.

The overall objective is 
\begin{equation*}
    \mathcal{L}_{ssl}=\mathcal{L}_s+\lambda_u\mathcal{L}_u+\lambda_f\mathcal{L}_f,
\end{equation*}
where $\lambda_u$ and $\lambda_f$ are weights for $\mathcal{L}_u$ and $\mathcal{L}_f$, respectively.

\input{table/ssl}

\subsection{Insights from information interplay}

For a batch of labeled data $\{(x_i, y_i)\}_{i=1}^{B} \in \mathcal{D}_L$, $h$ extracts feature representations $f \in \mathbb{R}^{B\times D}$. According to Neural Collapse theory, the representation of each class center aligns with the classifier weight of that category, i.e., $V_i=W{y_i}$. For unlabeled data $\{u_i\}_{i=1}^{\mu B} \in \mathcal{D}_U$, sample features $f'$ are selected from $\mu B$ samples with pseudo-label probabilities exceeding $\tau$, i.e., $f'={f_i \in f \mid \mathbb{I}(\max(q_j) > \tau)}$. The corresponding class centers are $V'=W{y_i'}$, where $y_i'$ is the pseudo-label of $f'$.

\noindent\textbf{Maximizing mutual information.}
As depicted in \figureref{fig:tp_MIR}, the mutual information between a batch's data features $f$ and corresponding class weights $V$ increases during model training. To enhance this, an additional loss term is added to maximize their mutual information. For supervised learning, the final objective is:
\begin{equation*}
    \mathcal{L} = \mathcal{L}_{s} - \lambda_{mi}\cdot\text{MI}\left(\mathbf{G}(f),\mathbf{G}(V)\right).
\end{equation*}
For semi-supervised learning, the objective becomes:
\begin{equation*}
    \mathcal{L} = \mathcal{L}_{ssl} - \lambda_{mi}\cdot\text{MI}\left(\mathbf{G}(f'),\mathbf{G}(V')\right),
\end{equation*}
where $\lambda_{mi}$ is the weight for mutual information.

\noindent\textbf{Minimizing entropy difference.} As shown in \figureref{fig:tp_MIR}, the disparity in information entropy between data features $f$ and category weights $V$ diminishes alongside accuracy improvements during training. An auxiliary loss is introduced to reduce this entropy difference further. For supervised learning, the objective is:
\begin{equation*}
    \mathcal{L} = \mathcal{L}_{s} + \lambda_{id}\cdot\left|\text{H}(\mathbf{G}(f))-\text{H}(\mathbf{G}(V))\right|.
\end{equation*}
For semi-supervised learning, this shifts to:
\begin{equation*}
    \mathcal{L} = \mathcal{L}_{ssl} + \lambda_{id}\cdot\left|\text{H}(\mathbf{G}(f'))-\text{H}(\mathbf{G}(V'))\right|,
\end{equation*}
where $\lambda_{id}$ is the weight for entropy difference.

\subsection{Performances on supervised and semi-supervised learning}

To ensure a fair comparison between our proposed method and existing techniques, we carefully designed experiments based on prior research. TorchSSL~\cite{zhang2021flexmatch}, a comprehensive codebase supporting various semi-supervised and supervised learning methods, served as our foundation. This enabled effective implementation and evaluation on well-known datasets like CIFAR-10, CIFAR-100, and STL-10. For supervised learning, our unique loss components were applied to labeled data, facilitating the computation of mutual information and entropy difference losses. In semi-supervised learning, these loss components were extended to unlabeled data, enhancing the calculation of these metrics. We employed an SGD optimizer with a momentum of 0.9, a weight decay of $5e^{-4}$, and an initial learning rate of 0.03, adjusted via cosine annealing. Performance metrics were reported over multiple seed runs. Batch sizes were set to 64 for a total of 1,048,000 iterations. WideResNet-28-2, WideResNet-28-8, and WideResNet-37-2 architectures were chosen for CIFAR-10, CIFAR-100, and STL-10, respectively.

\input{table/fullysupervised}

By incorporating mutual information and entropy difference constraints into the loss function, we achieved consistent performance improvements. \tableref{tab:semi} and \tableref{tab:fullysupervised} present the results for semi-supervised and supervised learning, respectively. In supervised learning, these constraints led to slight performance gains, likely due to the sufficient information constraints provided by labeled data. However, in semi-supervised learning, maximizing mutual information and minimizing entropy yielded the best or second-best performance in most scenarios. Notably, our method consistently outperformed the baseline FreeMatch across various settings, demonstrating its effectiveness in leveraging additional information constraints in low-labeled data scenarios.

%% file: table/ssl.tex
\begin{table*}[t]
\centering
\caption{Error rates (100\% - accuracy) on CIFAR-10/100, and STL-10 datasets for state-of-the-art methods in semi-supervised learning. Bold indicates the best performance, and underline indicates the second best.}
\resizebox{\textwidth}{!}{%
\begin{tabular}{l|ccc|cc|ccc}
\toprule
Dataset & \multicolumn{3}{c|}{CIFAR-10} & \multicolumn{2}{c|}{CIFAR-100}& \multicolumn{2}{c}{STL-10} \\ 
\cmidrule{1-1}\cmidrule(lr){2-4}\cmidrule(lr){5-6}\cmidrule{7-8} 
\# Label & 10 & 40 & 250 & 400   &2500& 40 & 1000\\ 
\cmidrule{1-1}\cmidrule(lr){2-4}\cmidrule(lr){5-6}\cmidrule{7-8}
$\Pi$ Model \cite{rasmus2015semi} & 
79.18{\scriptsize $\pm$1.11} &
74.34{\scriptsize $\pm$1.76} & 46.24{\scriptsize $\pm$1.29} & 86.96{\scriptsize $\pm$0.80}  & 58.80{\scriptsize $\pm$0.66} & 74.31{\scriptsize $\pm$0.85} & 32.78{\scriptsize $\pm$0.40} \\
Pseudo Label \cite{lee2013pseudo} & 80.21{\scriptsize $\pm$ 0.55} & 74.61{\scriptsize $\pm$0.26} & 46.49{\scriptsize $\pm$2.20} & 87.45{\scriptsize $\pm$0.85}  & 57.74{\scriptsize $\pm$0.28} & 74.68{\scriptsize $\pm$0.99} & 32.64{\scriptsize $\pm$0.71} \\
VAT \cite{miyato2018virtual} & 79.81{\scriptsize $\pm$ 1.17} & 74.66{\scriptsize $\pm$2.12} & 41.03{\scriptsize $\pm$1.79} & 85.20{\scriptsize $\pm$1.40}  & 48.84{\scriptsize $\pm$0.79} & 74.74{\scriptsize $\pm$0.38} & 37.95{\scriptsize $\pm$1.12} \\
MeanTeacher \cite{tarvainen2017mean} & 76.37{\scriptsize $\pm$ 0.44} & 70.09{\scriptsize $\pm$1.60} & 37.46{\scriptsize $\pm$3.30} & 81.11{\scriptsize $\pm$1.44}  & 45.17{\scriptsize $\pm$1.06} & 71.72{\scriptsize $\pm$1.45} & 33.90{\scriptsize $\pm$1.37} \\
MixMatch \cite{berthelot2019mixmatch} & 65.76{\scriptsize $\pm$ 7.06} & 36.19{\scriptsize $\pm$6.48} & 13.63{\scriptsize $\pm$0.59} & 67.59{\scriptsize $\pm$0.66}  & 39.76{\scriptsize $\pm$0.48} & 54.93{\scriptsize $\pm$0.96} & 21.70{\scriptsize $\pm$0.68} \\
ReMixMatch  \cite{berthelot2019remixmatch} & 20.77{\scriptsize $\pm$ 7.48} & 9.88{\scriptsize $\pm$1.03} & 6.30{\scriptsize $\pm$0.05} & 42.75{\scriptsize $\pm$1.05}  & 26.03{\scriptsize $\pm$0.35} & 32.12{\scriptsize $\pm$6.24} & 6.74{\scriptsize $\pm$0.17}\\
UDA \cite{xie2020unsupervised} & 34.53{\scriptsize $\pm$ 10.69} & 10.62{\scriptsize $\pm$3.75} & 5.16{\scriptsize $\pm$0.06} & 46.39{\scriptsize $\pm$1.59}  & 27.73{\scriptsize $\pm$0.21} & 37.42{\scriptsize $\pm$8.44} & 6.64{\scriptsize $\pm$0.17} \\
FixMatch \cite{sohn2020fixmatch} & 24.79{\scriptsize $\pm$ 7.65} & 7.47{\scriptsize $\pm$0.28} & 5.07{\scriptsize $\pm$0.05} & 46.42{\scriptsize $\pm$0.82}  & 28.03{\scriptsize $\pm$0.16} & 35.97{\scriptsize $\pm$4.14} & 6.25{\scriptsize $\pm$0.33} \\
Dash \cite{xu2021dash} & 27.28{\scriptsize $\pm$ 14.09} & 8.93{\scriptsize $\pm$3.11} & 5.16{\scriptsize $\pm$0.23} & 44.82{\scriptsize $\pm$0.96}  & 27.15{\scriptsize $\pm$0.22} & 34.52{\scriptsize $\pm$4.30} & 6.39{\scriptsize $\pm$0.56} \\
MPL \cite{pham2021meta} & 23.55{\scriptsize $\pm$ 6.01} & 6.93{\scriptsize $\pm$0.17} & 5.76{\scriptsize $\pm$0.24} & 46.26{\scriptsize $\pm$1.84}  & 27.71{\scriptsize $\pm$0.19} & 35.76{\scriptsize $\pm$4.83} & 6.66{\scriptsize $\pm$0.00} \\

FlexMatch \cite{zhang2021flexmatch} & 13.85{\scriptsize $\pm$ 12.04} & 4.97{\scriptsize $\pm$0.06} & 4.98{\scriptsize $\pm$0.09} & 39.94{\scriptsize $\pm$1.62}  & 26.49{\scriptsize $\pm$0.20} & 29.15{\scriptsize $\pm$4.16} & 5.77{\scriptsize $\pm$0.18} \\
FreeMatch \cite{wang2022freematch} & {8.07{\scriptsize $\pm$ 4.24}} & {4.90{\scriptsize $\pm$0.04}} & {4.88{\scriptsize $\pm$0.18}} & {37.98{\scriptsize $\pm$0.42}}  & 26.47{\scriptsize $\pm$0.20} & {15.56{\scriptsize $\pm$0.55}} & {5.63{\scriptsize $\pm$0.15}} \\
OTMatch \cite{tan2023otmatch} & {4.89{\scriptsize $\pm$ 0.76}} & {4.72{\scriptsize $\pm$0.08}} & {4.60{\scriptsize $\pm$0.15}} & {37.29{\scriptsize $\pm$0.76}}  & 26.04{\scriptsize $\pm$0.21} & \textbf{12.10{\scriptsize $\pm$0.72}} & {5.60{\scriptsize $\pm$0.14}} \\

SoftMatch \cite{chen2023softmatch} & {4.91{\scriptsize $\pm$ 0.12}} & {4.82{\scriptsize $\pm$0.09}} & \textbf{4.04{\scriptsize $\pm$0.02}} & {37.10{\scriptsize $\pm$0.07}}  & 26.66{\scriptsize $\pm$0.25} & {21.42{\scriptsize $\pm$3.48}} & {5.73{\scriptsize $\pm$0.24}} \\

\cmidrule{1-1}\cmidrule(lr){2-4}\cmidrule(lr){5-6}\cmidrule{7-8}

FreeMatch + MAX MI (Ours) & \underline{4.87{\scriptsize $\pm$ 0.66}} & \underline{4.66{\scriptsize $\pm$ 0.13}} & \underline{4.56{\scriptsize $\pm$ 0.15}} & \textbf{36.41{\scriptsize $\pm$ 1.91}}   & \textbf{25.77{\scriptsize $\pm$ 0.35}} & {16.61{\scriptsize $\pm$ 1.19}} & \textbf{{5.24 \scriptsize $\pm$ 0.17}} \\

FreeMatch + MIN HD (Ours) & \textbf{4.69{\scriptsize $\pm$ 0.16}} & \textbf{4.63{\scriptsize $\pm$ 0.25}} & {4.60{\scriptsize $\pm$ 0.15}} & \underline{37.31{\scriptsize $\pm$ 1.96}}   & \underline{25.79{\scriptsize $\pm$ 0.41}} & \underline{14.93 {\scriptsize $\pm$ 3.28}} & \underline{{5.30 \scriptsize $\pm$ 0.18}} \\
\bottomrule
\end{tabular}%
}
\label{tab:semi}
\end{table*}

%% file: table/fullysupervised.tex
\begin{table}[t]
\centering
\caption{Results for fully supervised learning}

\begin{tabular}{lcc}
\toprule
Method         & CIFAR-10 & CIFAR-100 \\ \midrule
Fully supervised &          95.35&           80.77\\
Ours (MAX MI)       &          95.52&           80.81\\
Ours (MIN HD)       &          \textbf{95.57}&          \textbf{80.96}\\ \bottomrule

\end{tabular}
\label{tab:fullysupervised}
\end{table}

%% file: content/conclusion.tex
\section{Conclusion}

In conclusion, we introduce matrix information theory as an analytical tool for analyzing neural networks. Leveraging the properties of matrix information entropy, we propose a novel Cross-Modal Alignment (CMA) loss. This loss function optimizes the fine-tuning process of cross-modal pre-trained models by utilizing the matrix information of representations from different modalities within the same category. CMA effectively enhances the cross-modal alignment capabilities of pre-trained models and improves their overall performance.

Additionally, we have made significant advancements in understanding the dynamics of supervised learning by integrating matrix information theory with Neural Collapse principles. Specifically, we observed changes in the matrix information entropy of sample representations and classification head weights during supervised learning. Our findings reveal that matrix information entropy alone is insufficient to fully describe the Neural Collapse phenomenon. To address this, we propose two novel metrics: the Matrix Mutual Information Rate (MIR) and the Matrix Entropy Difference Rate (HDR). These metrics provide deeper insights into the interplay between data representations and classification head vectors, serving as innovative tools for understanding neural network dynamics.

Through rigorous theoretical and empirical analyses, we demonstrate the effectiveness of MIR and HDR in explaining various neural network phenomena, including grokking, and their utility in enhancing training dynamics. Incorporating these metrics as loss functions in supervised and semi-supervised learning yields promising results, highlighting their potential to improve model performance and training efficiency. This study not only contributes to the field of machine learning by introducing new analytical tools but also showcases the application of matrix information theory in optimizing supervised learning algorithms.